\documentclass{article}

\usepackage{arxiv}

\usepackage[utf8]{inputenc} 
\usepackage[T1]{fontenc}    
\usepackage{hyperref}       
\usepackage{url}            
\usepackage{booktabs}       
\usepackage{amsfonts}       
\usepackage{nicefrac}       
\usepackage{microtype}      
\usepackage{lipsum}		
\usepackage{graphicx}

\usepackage{amsthm}
\usepackage{amsmath,amssymb,amsfonts}
\usepackage{multirow}
\usepackage{makecell}

\newtheorem{prop}{Proposition}
\usepackage{subfigure}
\usepackage{textcomp}
\usepackage{xcolor}
\usepackage{bm}
\usepackage{algorithm}
\usepackage{algorithmic}
\newcommand{\ie}{\textit{i.e.}, }

\title{DRNet: A Decision-Making Method for Autonomous Lane Changing with Deep Reinforcement Learning}

\author{ Kunpeng Xu \\
	Department of Computer Science\\
	Université de Sherbrooke\\
	Québec, Canada \\
	\texttt{kunpeng.xu@usherbrooke.ca} \\
	\And
	Lifei Chen \\
	College of Computer and Cyber Security\\
	Fujian Normal University\\
	Fuzhou, China \\
	\texttt{clfei@fjnu.edu.cn} \\
 \And
    Shengrui Wang \\
	Department of Computer Science\\
	Université de Sherbrooke\\
	Québec, Canada \\
	\texttt{shengrui.wang@usherbrooke.ca} \\ 
}

\renewcommand{\shorttitle}{\textit{arXiv} Template}

\hypersetup{
pdftitle={A template for the arxiv style},
pdfsubject={q-bio.NC, q-bio.QM},
pdfauthor={David S.~Hippocampus, Elias D.~Striatum},
pdfkeywords={First keyword, Second keyword, More},
}

\begin{document}
\maketitle

\begin{abstract}
 Machine learning techniques have outperformed numerous rule-based methods for decision-making in autonomous vehicles. Despite recent efforts, lane changing remains a major challenge, due to the complex driving scenarios and changeable social behaviors of surrounding vehicles. To help improve the state of the art, we propose to leveraging the emerging \underline{D}eep \underline{R}einforcement learning (DRL) approach for la\underline{NE} changing at the \underline{T}actical level. To this end, we present ``DRNet", a novel and highly efficient DRL-based framework that enables a DRL agent to learn to drive by executing reasonable lane changing on simulated highways with an arbitrary number of lanes, and considering driving style of surrounding vehicles to make better decisions. Furthermore, to achieve a safe policy for decision-making, DRNet incorporates ideas from safety verification, the most important component of autonomous driving, to ensure that only safe actions are chosen at any time. The setting of our state representation and reward function enables the trained agent to take appropriate actions in a real-world-like simulator. Our DRL agent has the ability to learn the desired task without causing collisions and outperforms DDQN and other baseline models.
\end{abstract}


\section{Introduction}
The past few years have seen a rapid increase in interest in autonomous vehicles, which are widely regarded as one of the most important factors in improving transportation systems. For example, they have the potential to eliminate traffic accidents primarily resulting from improper operations by human drivers \cite{Fagnant2015Preparing}. Building such autonomous systems has been an active area of research, due to its high importance for creating and maintaining safer and more efficient road networks \cite{Cosgun2017Towards}. One of the fundamental skills that an autonomous vehicle must possess is the ability to perform lane-changing maneuvers, which is especially critical in the presence of multi-lane and fast-moving traffic. The lane-changing maneuver can be a demanding task because the vehicle needs to alertly watch the leading vehicle in its own lane and surrounding vehicles in the target lane, and to perform proper actions according to the potential adversarial or cooperative reactions demonstrated by those surrounding vehicles \cite{Urmson2009Autonomous}. 

A possible way of handling this complexity is by separating the planning task into high-level decision-making and maneuver execution \cite{Mirchevska2018High}. The decision-making layer can be viewed as a tactical level function which issues a lane-change command such as change lanes or stay in Lane, while in the maneuver execution layer, the specific motion is planned and executed. In the example shown in Fig.1, the high-level decision-making layer gives the order to change lanes to the left, and the maneuver execution layer performs operational control to coordinate the longitudinal and lateral movements for a safe, smooth and efficient lane change maneuver. The focus of this paper is on high-level policy for tactical decision-making. A study has shown that around 10 percent of all freeway crashes are caused by lane-change intention: a bad decision leads at best to congestion and at worst to accidents \cite{Jula2000Collision}.

\begin{figure}[t]
	\centerline{\includegraphics[scale=0.6]{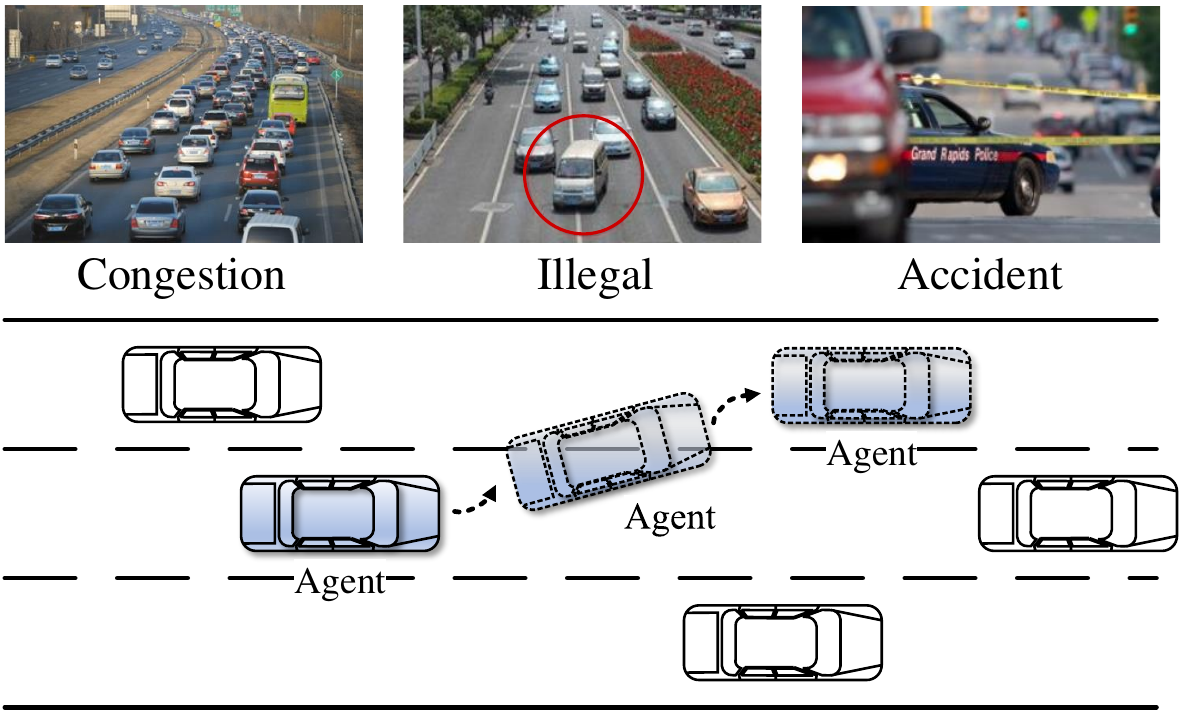}}
\caption{Lane-Changing Maneuver}
\label{fig1}
\end{figure}

Existing rule-based or statistical approaches require that the precise model be specified with a great deal of hand design and tuning. Such approaches may work only on a subset of traffic cases \cite{ardelt2012highly,suh2018stochastic}, and are not suitable for accomplishing many challenging real-life tasks. This is because our application problem has many objectives to achieve simultaneously, including reasoning about interactions with other vehicles, forming long-term strategy, safety analysis, and timely adaptation to traffic scenarios \cite{xu2018self}. For instance, the autonomous vehicle should be able to consider and predict the social behaviors -- \ie the driving styles and intentions of surrounding vehicles, in order to make a more rational decision. When all surrounding vehicles are treated equally, the lane-changing system is obliged to make very conservative decisions to reduce all possible collision risks \cite{xu2022multi}. However, based on real-world observations, the social behaviors are diverse and usually react differently to stimulus, which creates complex driving styles. As a result, understanding the underlying driving behaviours and styles of surrounding vehicles becomes necessary in enabling a safe and efficient lane-changing decision-making.

Recent progress in research on deep reinforcement learning (DRL) paves the way for possible alternative solutions which have outperformed humans on several game-playing tasks (e.g., Atari Games); it has also allowed revolutionary advances in the area of autonomous driving \cite{chen2019attention,ye2020automated}. However, directly applying any of the existing DRL methods to lane changing is still quite challenging because of the complicated environment under consideration, involving practical issues like how to obtain complete state information, predict the social behaviors, ensure the safety of the lane-changing policy, and adapt to new traffic scenarios quickly. 
The main contributions of this paper can be summarized as follows:

\begin{itemize}
\item[1)] We present the novel decision-making framework ``DRNet", a deep neural network (DNN) capable of learning a policy for making decisions at a tactical level. DRNet incorporates a prioritized replay buffer and a proposed Action-Subspace method, which makes decisions more efficient and safer, while using a convolutional neural network (CNN) to extract spatial features as the input.
\item[2)] We incorporates a prioritized replay buffer to makes decisions more efficient and explicitly provide an initialization method for the replay buffer. We also find an appropriate set of hyperparameters that achieves the best performance in terms of experience replay buffer size, discount factor, number of CNN hidden layers and neurons in fully connected layers.
\item[3)] We extensively evaluate the effectiveness and superiority of the proposed algorithm, by comparison with the state-of-the-art approach DDQN and rule-based method, in terms of diverse metrics.
\end{itemize}

The remainder of this paper is organized as follows: Section II reviews related research and introduces necessary background for DRL. Section III presents the system model and problem formulation. Our proposed approach DRNet is described in Section IV. Section V presents the simulation results. Finally, discussion and conclusions are given in Section VI.
\section{Related Work}
\subsection{Lane-Changing Decision-making}
Lane-changing decision-making is a challenging problem due to multiple surrounding vehicles and complex road environments. Existing methods for lane-changing decision-making can be divided into three categories: rule-based methods, statistical model-based methods, and machine learning-based methods.

A straightforward decision-making approach utilizes manually designed rules and relies on state machines to switch between predefined decision behaviors. The CMU’s team developed an autonomous vehicle, ``Boss", which determines the triggering of lane-changing behavior using thresholding and binary decisions \cite{Urmson2009Autonomous}. Similarly, in \cite{ardelt2012highly}, a hybrid, deterministic state machine is developed to define the superordinate driving behavior, with a decision tree that is used as a hierarchical decision-making process. However, while these methods may perform well under predefined operating conditions, they may be prone to failure when unexpected situations are encountered.

Some statistical models are also used for decision-making \cite{6107579,Gray2013Robust}, although the intrinsic utility of these models is for measuring how well the training data fit a model. To recognize lane-changing intention, a method which combined the Hidden Markov model (HMM) and Bayesian filtering techniques was proposed in \cite{Li2015Lane}. Other researchers, such as \cite{7225752,suh2018stochastic}, adopted the model predictive control (MPC) method, which employs a dynamic vehicle model to predict future states and determines an optimal control sequence for every step, aiming to optimize a predefined performance index that satisfies the control and state constraints. The major weakness of these methods is that they have difficulty covering all scenarios, due to the complexity of real-world traffic situations and limitations on training data. Furthermore, choosing the appropriate model and setting its parameters correctly is also a difficult task. 

The evolution of computing power and the increasing ubiquity of data have paved the way for applying machine learning-based approaches to decision-making. Machine learning-based approaches provide the possibility for vehicle adaptive learning and the improvement of future learning strategies on the basis of accumulated experience. The seminal work in this regard \cite{Pomerleau-1989-15721} used a neural network to map images directly to steering angles. Building on \cite{Pomerleau-1989-15721}, many other methods such as those in \cite{Chen2017End,wang2018reinforcement} have been proposed to obtain the proper steering angle. The work reported in \cite{Mirchevska2018High,hoel2018automated,ye2020automated} investigated the lane-changing problem for an autonomous vehicle driven by DRL. However, the state in these models was defined as a one-dimensional vector of the relative positions and speeds of surrounding vehicles, and the absolute velocity of the ego vehicle, which, from a statistical point of view, results in a loss of environment information. Also, none of them explicitly considers the efficiency of relearning when the scenes and vehicle distribution change.

The drawbacks of the existing work can be summarized as follows:
\begin{itemize}
\item[1)] Failure to tackle changing environments: A common problem with most existing methods is that they target one specific driving case. The disadvantages of these methods become obvious due to their poor generalization ability when dealing with traffic scenarios that have not been considered. It is depressing that most methods need to be rebuilt or relearned if the external scene and vehicle distribution changes significantly.
\item[2)] Lack of ``foresight": The traditional lane-changing decision-making approach easily gets bogged down in local optima such as being boxed in by surrounding vehicles. Autonomous vehicles should have the ability to reason about interactions with other agents and form an efficient long-term strategy.
\item[3)] Insensitivity to unsafe decision-making: An unsafe decision can cause a collision or increase complexity. Safety at all times is especially important when performing learning in real traffic, with imminent danger for other traffic participants.
\end{itemize}
\subsection{DRL}

Reinforcement learning (RL) has proven quite successful recently in solving complex sequential decision-making problems. It addresses the problem of an agent interacting with a local environment \textit{E} in discrete timeslots. As shown in Fig.\ref{figRL}, the environment block is a physical or dynamic system, e.g., a traffic environment, and the RL agent is a controller that continuously interacts with the environment. The iteration of interactions starts when the RL agent receives the state measurement $(s_t)$ of the environment, and then the agent responds by selecting an action $(a_t)$. After the environment executes the action, it generates the corresponding reward $(r_t)$ and a new state $(s_{t+1})$. On the agent side, action selection is given by a policy $\pi$ that defines a probability distribution over $a_t$ for each state $s_t$. This is usually modeled as a Markov decision problem (MDP) by defining a reward function $r(s_t, a_t)$. The goal of RL is to learn an optimal policy which maps from environmental states to agent’s actions by maximizing the cumulative reward of actions taken by the environment. Thus, the return from a state is replaced by the sum of discounted future reward, $R_t=\sum_{i=t}^T\gamma^{(i-t)}r(s_t, a_t)$, with a discount factor $\gamma \in [0, 1]$. We estimate the expected return $Q$ with given $s_t$ and $a_t$ by the following value function:
\begin{equation}
Q^\pi(s_t, a_t)=\mathbb{E_\pi}[R_t|s_t,a_t] \label{eq1}
\end{equation}
 
\begin{figure}[t]
\centerline{\includegraphics[scale=0.6]{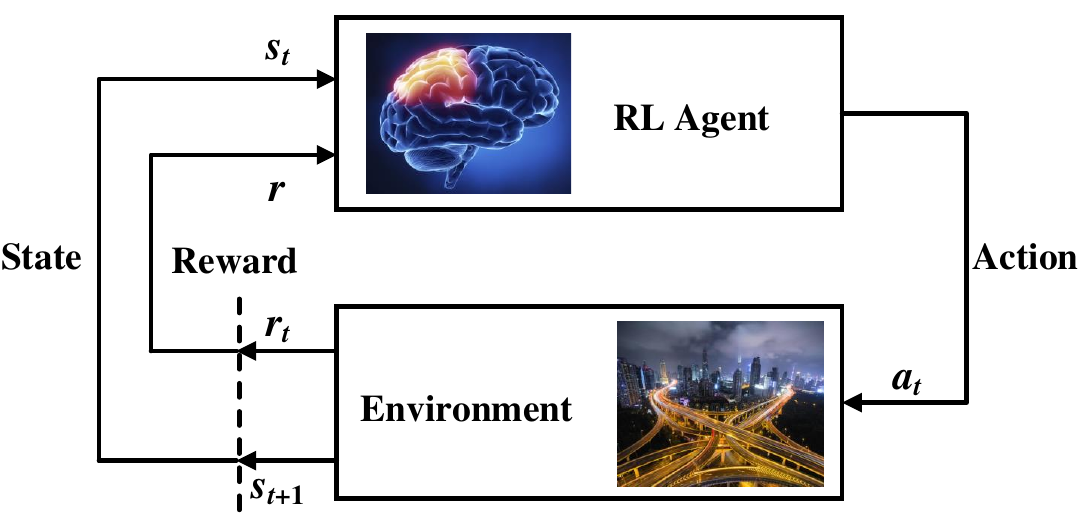}}
\caption{ Framework Diagram for Reinforcement Learning}\label{figRL}
\end{figure}

DRL can be considered as the “deep” version of RL, which takes advantage of the power of DNNs for representation learning. The earliest popular solution of DRL was Deep Q-Networks (DQN \cite{mnih2015human}), which can learn from a set of sequential frames to play many Atari games at a human performance level. To represent large state or action spaces in the learning of $Q$ values, DQN successfully combines RL with a DNN (e.g., a CNN) to approximate $Q^\pi(s_t,a_t)$. The DQN model is optimized by using stochastic gradient descent to minimize the loss $L$, as follows:
\begin{equation}
L(\theta_i)=\mathbb{E}\left[\left( r_t+\gamma maxQ{(s_{t+1},a_{t+1};\theta_i^-)}-Q{(s_t,a_t;\theta_i)}\right)^2\right]  \label{eq2} 
\end{equation}
where $\theta_i$ is the parameter of the online network at iteration $i$, and the term $\theta_i^-$ represents the parameters used to compute the target network, which is a periodic copy of the online network. The agent’s experiences $e_t=(s_t, a_t, r_t, s_{t+1})$ at each timeslot, stored in a data set $D_M={e_1, e_2,..., e_M}$, are used to train the Q-network. At iteration $i$, a minibatch of experiences $m$ is randomly sampled from $D_M$ to update the parameters $\theta_i$.

DeepMind presented Double DQN (DDQN) to overcome the overestimation that occurs in both Q-learning and DQN algorithms. This can be fixed by decomposing the max operation into action selection and action evaluation, which is the main idea of DDQN. The target used by DDQN is then
\begin{equation}
Y_t^{\rm{DDQN}}=r_t+\gamma Q{(s_{t+1}, \mathop{\arg\max}_{a_{t+1}} Q{(s_{t+1},a_{t+1}};\theta_i);\theta_i^- )}
\end{equation}

During learning, parameters are updated on samples from replay memory. The Q-learning update at iteration $i$ uses the following loss function:
\begin{equation}
L(\theta_i)=\mathbb{E}\left[\left(Y_t^{\rm{DDQN}}-Q{(s_t,a_t;\theta_i)}\right)^2\right]
\end{equation}

Like DQN, DDQN has two network architectures. The main network is used for evaluating the greedy policy, while the target network is used to estimate its value. Results show that DDQN outperforms DQN, and it is thus considered one of the state-of-the-art DRL approaches. However, directly applying DDQN to the problem under consideration does not work well either, since we are dealing with a multi-task (lane changing, safety verification and fast learning) scenario.

\section{System Model and Problem Formulation}
Consider an autonomous vehicle which must make an automated lane-changing decision in a target area. The area comprises three one-way lanes ($L$=3), with other traffic participants the autonomous vehicle has to avoid. Without loss of generality, we assume that there is a set $\mathcal{P}\triangleq \{p|p=1,2,,...,P\}$ of participants. Each participant is controlled by a simulation system using a complex set of rules based on expert knowledge. Let traffic density be $D_{lane} = \left\lbrace (1,d_1), (2,d_2), (3,d_3)\right\rbrace $ for lanes 1 to 3 (lane 1 is the leftmost lane). We assume that a road driving task lasts for $T$ timeslots and define the sensing capability of the autonomous vehicle as its sensing range $U$. Initially, the autonomous vehicle is deployed in a random lane $l$ with minimum velocity. Then, at each timeslot $t$, the vehicle makes the decision $a$, which is executed by the low-level controller. 

We model the dynamic evolution problem where the autonomous vehicle interacts with the environment, including surrounding vehicles and lanes, as an MDP. Roughly speaking, an MDP involves a decision agent that repeatedly observes the current state $s$ of the environment and takes an action $a$ among the available actions allowed in that state. The agent then switches to a new state $s'$ and obtains a reward $r$, all within a timeslot $t$. We define the MDP based on the following components:

1) \textbf{State} 

The state of the autonomous vehicle is constructed via vehicle sensors.  The information data provided by the sensors includes the positions and velocities of the autonomous vehicle and other participants. These data are combined and traffic conditions are represented in a grid form. To capture the relative motion of the traffic, we input a history of the occupancy grid from previous timeslots. Specifically, the state $s_t$ contains a three-layer binary grid at timeslots $t-2, t-2,$ and $ t$. As shown in Fig.\ref{state}, a layer of the state is composed of participants and the autonomous vehicle, indicating the lane position.
\begin{figure}[t]
\centerline{\includegraphics[scale=0.6]{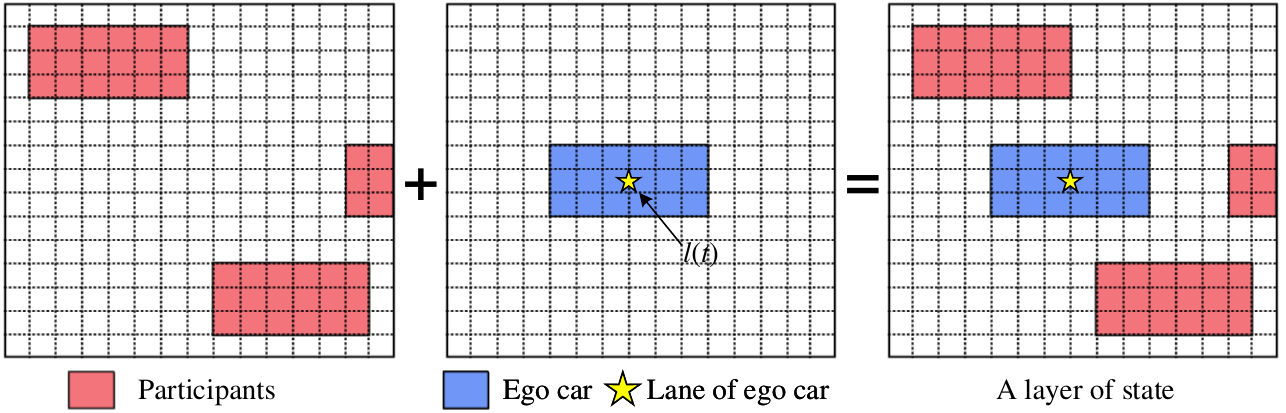}}
\caption{A layer of state representation}
\label{state}
\end{figure}

\begin{figure*}[t]
\centerline{\includegraphics[scale=0.45]{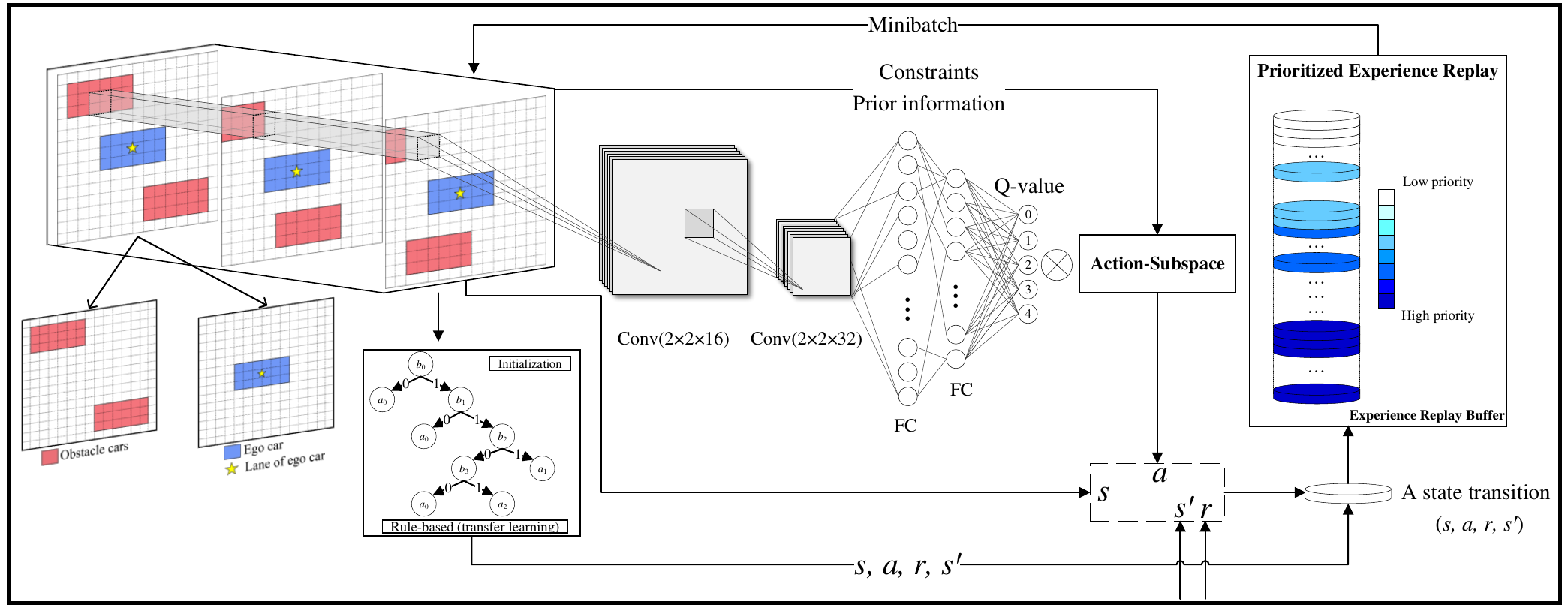}}
\caption{DRNet Framework }
\label{fig5}
\end{figure*}

2) \textbf{Action} 

For lane-change maneuvers, we break down the high-level decisions into the 5 actions given in Table.\ref{tab1}, which can be taken in any timeslot. The action space of the agent is thus discrete.
\begin{table}[htbp]
\caption{ Set of Actions}
\begin{center}
\resizebox{0.8\textwidth}{!}{
\begin{tabular}{cc}
\hline
\textbf{Decision}& \textbf{Action} \\
\hline
$a_0$& \textit{Stay in the current lane and maintain current velocity}  \\
\hline
$a_1$& \textit{Change lanes to the Left}\\
\hline
$a_2$& \textit{Change lanes to the Right}\\
\hline
$a_3$& \textit{Accelerate}\\
\hline
$a_4$& \textit{Decelerate}\\
\hline
\end{tabular}}

\label{tab1}
\end{center}
\end{table}

We have opted for discrete actions, rather than low-level controls like steering, because modularized systems have been reported to perform better in autonomous driving than end-to-end systems \cite{yurtsever2020survey}.

3) \textbf{Reward} 

Safety: The most important evaluation factor in the lane-change decision-making process is safety. The design of the reward function for safety is primarily considered in this study from two perspectives: avoiding collisions between the agent and other participants, and setting different desired lane-changing distances to represent different driving styles.

To avoid collisions with surrounding vehicles, a large penalty is imposed on the agent if a lane-change decision results in a collision. The specific design is as follows:
\begin{equation}
r_{c}=-\lambda_c \cdot c_t, \quad c_t=\left\{  
\begin{array}{lr}  
             1,  \textit{ if collision}\\
             0,  \textit{ else}   
             \end{array}  
\right.
\end{equation}
where $\lambda_c>0$ is a weight coefficient for the collision between the agent and other participants and $c_t$ is the detection mark for the collision.

Subsequently, different driving styles can be characterized by setting different desired lane-changing distances from the vehicles ahead. When the distance between the agent and the vehicle in front is less than the desired distance, the penalty value will gradually increase with the increase in the desired distance. Formally, we define the following penalty according to the distance:
\begin{equation}
r_d=-\lambda_d |d-d_{des}|\cdot l_t, \quad l_t=\left\{  
\begin{array}{lr}  
             1,  \textit{ lane changing}\\
             0,  \textit{ else}   
             \end{array}  
\right.
\end{equation}
where $\lambda_d>0$ is a normalizing coefficient for lane-changing distance to make the range of $\lambda_d \left| d-d_{des}\right|$ to $\left[0,1\right]$. $d$ and $d_{des}$ describe the distance of the agent from the vehicle in front and the desired lane-change distance, respectively. $l_t$ is the detection mark for lane changing, to avoid meaningless lane-changing behavior. A conservative driving style is one in which the agent makes a lane-changing decision only when there is a wide gap between the agent and the vehicle in front, while an aggressive driving style is one in which the agent is instructed to change lanes even when the distance between the agent and the vehicle in front is very small. Thus, we can use the reward function to describe various driving styles by setting different desired lane-changing distances.

Efficiency: For efficiency, the autonomous vehicle must try to meet the requirements of driving as fast as possible without exceeding the appropriate speed limit, and without changing lanes frequently. To reward the above behavior, the reward function is defined as follows:

\begin{equation}
r_e=-\frac{\lambda_e \left|v_{t}-v_{des}\right|}{\beta_t}\label{eq3}
\end{equation}
where $v_{des}$ is the desired velocity (not the maximum speed limit) and $\lambda_e>0$ is a normalizing coefficient to make the range of $\lambda_e \left| v_{t}-v_{des}\right|$ to $\left[0,1\right]$. $\beta_t$ represents punishment for frequent lane changing. When lane-changing decisions are made in consecutive timeslots, $\beta_t$ is set to a constant between 0 and 1. Otherwise, it is assigned a value of 1.

Based on analysis and consideration of  safety and efficiency, the total reward function of the agent can be expressed as follows:
\begin{equation}
r=r_c+r_d+r_e
\end{equation}
\section{Proposed Solution: DRNet}
In this section, we present our proposed solution DRNet for policy-making in the context of the tactical lane-changing problem. The proposed neural network framework, shown in Fig.\ref{fig5}, can be divided into three parts.
\subsection{Feature Extraction and Driving style Identification}

Here we use an image with three layers to represent the observations at each timeslot, with the first two layers representing the history of 2 previous timeslots, as shown in the upper left part of Fig.\ref{fig5}. We utilized CNN to better exploit features in the image that can help the DRL agent (autonomous vehicle) make better decisions, while fully considering (a) the relative positions of the DRL agent and other participants, (b) the distribution of surrounding traffic participants, and (c) the correlation of traffic conditions and the decision predictability of the DRL agent.

The three layers represent the historical and current information of traffic, which is crucial for understanding the current state and predicting future traffic behavior. We use the three layers to work on the classification of driving styles based on observable behaviors, such as the speed of the vehicle relative to others, lane-changing patterns, and braking behavior.

To identify the driving style of each vehicle, we we utilize a combination of Convolutional Neural Networks (CNNs) and Support Vector Machines (SVMs) that analyze the data represented in the image. The classification into aggressive, normal, and cautious driving styles is based on a set of predetermined criteria.

\paragraph{Convolutional Neural Networks (CNNs):}
Initially, a CNN is applied to extract relevant features from the image data, which encompasses spatial relations and movement patterns of vehicles. The CNN architecture comprises several convolutional layers followed by pooling layers. Mathematically, the convolution operation in layer $l$ can be represented as:
\begin{equation}
F_{l} = \sigma(W_{l} * F_{l-1} + B_{l})
\end{equation}
where \( F_{l-1} \) is the feature map from the previous layer, \( W_{l} \) and \( B_{l} \) are the weights and biases of the current layer, \( * \) denotes the convolution operation, and \( \sigma \) represents the non-linear activation function, such as ReLU.

\paragraph{Support Vector Machines (SVMs):}
The feature map obtained from the final layer of the CNN is then fed into an SVM for classification. The SVM is trained to classify the driving styles into three categories based on the extracted features. The decision function of the SVM can be represented as:
\begin{equation}
y = \text{sign}(\sum_{i=1}^{n} \alpha_i y_i K(x_i, x) + b)
\end{equation}
where \( y \) is the predicted class, \( x \) represents the input feature vector, \( \{x_i, y_i\}_{i=1}^{n} \) are the training samples with their corresponding labels, \( \alpha_i \) are the Lagrange multipliers, \( K \) is the kernel function, and \( b \) is the bias term. In our case, a Radial Basis Function (RBF) kernel is utilized for its effectiveness in handling non-linear data.

\paragraph{Integration into the DRL Framework}
The classification of driving styles is integrated into the overall Deep Reinforcement Learning (DRL) framework. The output from the SVM, indicating the driving style of each vehicle in the vicinity of the autonomous agent, is incorporated into the state representation. This enriched state information allows the DRL agent to make more informed decisions, taking into consideration not only the current traffic conditions but also the predicted behavior patterns of surrounding drivers.

\subsection{Initialization and Sampling Transitions}
The state transition --- the atomic unit of interaction in DRL, shown by the flat cylinders in Fig.\ref{fig5} --- includes the observation, the action and the reward given by the environment, which are stored in an experience replay buffer. Using a replay memory entails two design choices: which experiences to store, and which experiences to replay (and how to do so). Concerning storage, the DRL agent interacts with the environment and populates the experience. When this buffer is full, the system replaces the oldest transition with the new one. In the early period of learning, the experience replay buffer is full of unnecessary transitions because the DRL agent is starting from tabula rasa. This leads to long iterations of updates to make transitions meaningful \cite{Sutton1998Reinforcement}.

To address this, a transfer learning technique that integrates the rule-based method is embedded in DRNet to generate significant transitions in the replay memory initialization. As shown in the lower left part of Fig.\ref{fig5}, we first generate a decision tree (DT) for a simple lane-changing policy, in which only the first three actions $\left(  a_0, a_1, a_2 \right)  $ are considered. This decision tree is binary, in that every non-terminal node corresponds to a binary split on a discrete value True or False $\left\lbrace 1,0\right\rbrace $. We then use the DT to predict actions and transfer the transitions to the experience replay buffer. This DT is only used to initialize experience replay; subsequent transition updates are implemented by the neural network. With the help of initialization, the deep network can learn more quickly. Table.\ref{tab2} provides a list of non-terminal nodes.

\begin{table}[htbp]
\caption{List of Non-terminal Nodes }
\begin{center}
\begin{tabular}{cc}

\hline
\textbf{Node}& \textbf{Explantation} \\
\hline
$b_0$& \textit{Whether there is a vehicle in front of the ego car}  \\
\hline
$b_1$& \textit{ \makecell[c]{ Whether the ego car has a high speed \\ or is in a car-following state}}\\
\hline
$b_2$& \textit{Is the left lane empty?}\\
\hline
$b_3$& \textit{Is the right lane empty?}\\
\hline

\end{tabular}

\label{tab2}
\end{center}
\end{table}

As can be seen from the decision tree, when both left and right lanes meet the condition for lane changing, the left lane is preferred, on the assumption that vehicle density in the left lane is lower than in the right lane.

In our case, since the traffic scenarios are represented in a binary occupancy grid, the state space is enumerable, resulting in enumerable state transitions. To increase the independence of each state transition, the experience replay buffer is usually assigned a vast volume, so that there is a slim chance of a transition being chosen through random sampling in a minibatch. Furthermore, important transitions that record DRL agent obstacle avoidance or other ``smart" cases are hidden amid massive numbers of unwise cases like changing lanes early due to the anticipation of congestion ahead. Therefore, we include a prioritized experience replay mechanism\cite{Schaul2015Prioritized} to pick a minibatch of experience to make the most effective use of the replay memory for learning. In this way, each state transition is given a priority (TD-error) which will be higher if the state transition is more important, leading to a larger chance of being sampled. For each transition, we compute its TD-error as follows:
\begin{equation}
\delta_t=r_t+\gamma Q{(s_{t+1}, \mathop{\arg\max}_{a_{t+1}} Q{(s_{t+1},a_{t+1}}))}-Q{(s_t,a_t)}
\label{td}
\end{equation}

\subsection{Action-Subspace}
In a typical DQN, a mapping between states and Q-values associated with each action is learned. Then, a max (or soft-max) operator can be applied on the output layer of Q-values to pick the best action \cite{mnih2015human}. However, this technique is not capable of guaranteeing the safety of the agent at all times. To ensure safe lane-changing behavior, we propose a method, Action-Subspace, which is applied on the output Q-values before performing the max operation, as shown in Fig.\ref{fig5}. The direct effect of this is that when performing the max operation to choose the best action, we consider the Q-values associated with only a subspace of actions, which are dictated by prior knowledge and safety information from the observation \cite{xu2022data}.

At arbitrary timeslot $t$, we introduce a weight vector \textit{W} for the action set $A$, where $\textit{W}=<w_{0t},w_{1t},...w_{kt},...w_{4t}>, 0\leq k\leq 4$, satisfying
\begin{equation}
w_{kt}=\left\{  
\begin{array}{lr}  
             0,  \textit{ potentially dangerous action}\\
             1,  \textit{ safe action}   
             \end{array}  
\right.\label{eq4}
\end{equation} 
where $w_k$ represents the safety contribution of action $a_k$ for decision-making and can be determined by low-level control module information.
\begin{prop}[Safe Action Subspace]
The subspace of safe actions $A_{sub}(t)\subseteq A$  is defined as
\begin{large}
\begin{equation}
A_{sub}(t):=\{a_k\in A|w_{kt}=1, k=0,1,...,4\}
\nonumber
\end{equation}
\end{large}
\end{prop}
\begin{proof}	
Due to the weight vector $\textit{W}$ introduced to represent the safety importance of actions, the set of actions $A$ should be reformulated as $A^\ast=\{w_{0t}\cdot a_0, w_{1t}\cdot a_1,...,w_{kt}\cdot a_k,...,w_{4t}\cdot a_4  \}$. Combining this with Eq.\ref{eq4}, the soundness of Prop.1 is proven.
\end{proof}
The agent eventually learns a safe action subspace that provides benefits. For example, in the lane-changing problem, if the autonomous vehicle is already in the leftmost lane and approaches the vehicle ahead at a high speed, performing a Left or Accelerate action will result in traffic violations and may even lead to an accident. Therefore, $A_{sub}=\lbrace a_0, a_2, a_4\rbrace$, which excludes $a_1$ ($\textit{Change lanes to the left}$) and $a_3$ ($\textit{Accelerate}$) due to $w_1=w_3=0$. This enables us to incorporate prior knowledge about the environment directly into the learning process. Also, since we do not have to set up extra negative rewards and the agent does not explore these states, learning itself becomes faster and more data-efficient.  

In order to encounter as many diverse situations as possible during the transition update phase, we do not rely entirely on the positive samples collected from Action-Subspace. When the action corresponding to the maximum Q-value is outside the $A_{sub}$, it is also collected in a transition sample and a negative reward $r=-1$ is given to the agent.
\subsection{Algorithm Descriptions: DRNet}
Pseudocode for our approach DRNet is presented in Algorithm \ref{alg:A}. It includes the training process and environment interactions. After adequate training, the model (i.e., the parameters in those DNNs) is saved for testing.

Our solution works as follows. At the beginning, we initialize minibatch $k$, step-size $\eta$, replay period $K$, budget $T$, batch $H$ and buffer $B$ (Line 1-2). For the DRL agent, we initialize the online network $Q(s_t,a_t;\theta_i)$ and the target network $Q(s_t,a_t;\theta_i^-)$ with randomly initialized weights $\theta_i$ and $\theta_i^-$, respectively (Line 3). To improve learning efficiency, we use the decision tree for high-quality initialization of experience replay (Lines 4-7). The training process starts when the number of transitions collected in $B$ is sufficient for sampling. We first sample a prioritized minibatch of transitions by using prioritized experience replay (Line 12, and in Section IV.B). After training, the transition priorities are updated by calculating the new TD-error (Line 14) and the network weights then get updated (Line 17). Last, we determine the Action-Subspace based on the state information and choose the action $a_t$ with $\varepsilon$-greedy (Line 20, and in Section IV.C). Every time the replay buffer $B$ is full, the oldest experience will be removed (Lines 21-22).

\begin{algorithm}[t] 
\caption{Proposed Solution: DRNet}  
\label{alg:A}  
\begin{algorithmic}[1]  
\STATE {\textbf{Input:} minibatch $k$, step-size $\eta$, replay period $K$, budget $T$;}
\STATE{Initialize batch $H=\emptyset$, buffer $B=\emptyset$, $\triangle=0$;}
\STATE{Initialize the online network $Q(s_t,a_t;\theta_i)$ and target network $Q(s_t,a_t;\theta_i^-)$ with weights $\theta_i$ and $\theta_i^-$; }
\FOR{$j=1$ \textbf{to} $H$}
	\STATE{Observe $s_j$ and choose $a_j\sim\pi_\theta(s_j)$  by decision tree;}
	\STATE{Observe $s_{j+1}$, $r_j$;}
	\STATE{Store transition ($s_{j}, a_{j}, r_j, s_{j+1}$) in $H$; }
\ENDFOR
\FOR{episode, $i=1,2,...,M$}
	\FOR{$t=1$ \textbf{to} $T$}
		\IF{$t\equiv 0$ mod $K$}
			\STATE{Sample a prioritized minibatch of transitions, using prioritized experience replay;}
			\FOR{$m=1$ \textbf{to} $k$}
				\STATE{Update the priorities $\delta_m$ by calculating new TD-error, using Eq.\ref{td}; }
				\STATE{Accumulate weight-change $\triangle\leftarrow \triangle + \delta_m\cdot\nabla_{\theta}Q(s_m, a_m)$}
			\ENDFOR
			\STATE{Update weights $\theta_i\leftarrow \theta_i+\eta\cdot\triangle$, reset $\triangle=0$}
			\STATE{From time to time copy weights into target network $\theta_i^-\leftarrow\theta_i$}
	\ENDIF	
	\STATE{Determine the Action-Subspace based on the state information, and choose action $a_t\sim\pi_\theta(s_t)$ with $\varepsilon$-greedy}
	\IF{$B$ is full}
		\STATE{Remove oldest experience from replay buffer $B$}
	\ENDIF
	
	\ENDFOR

\ENDFOR
 
\end{algorithmic}  
\end{algorithm}

\section{Performance Evaluation}

\begin{figure}[t]
\centerline{\includegraphics[scale=0.6]{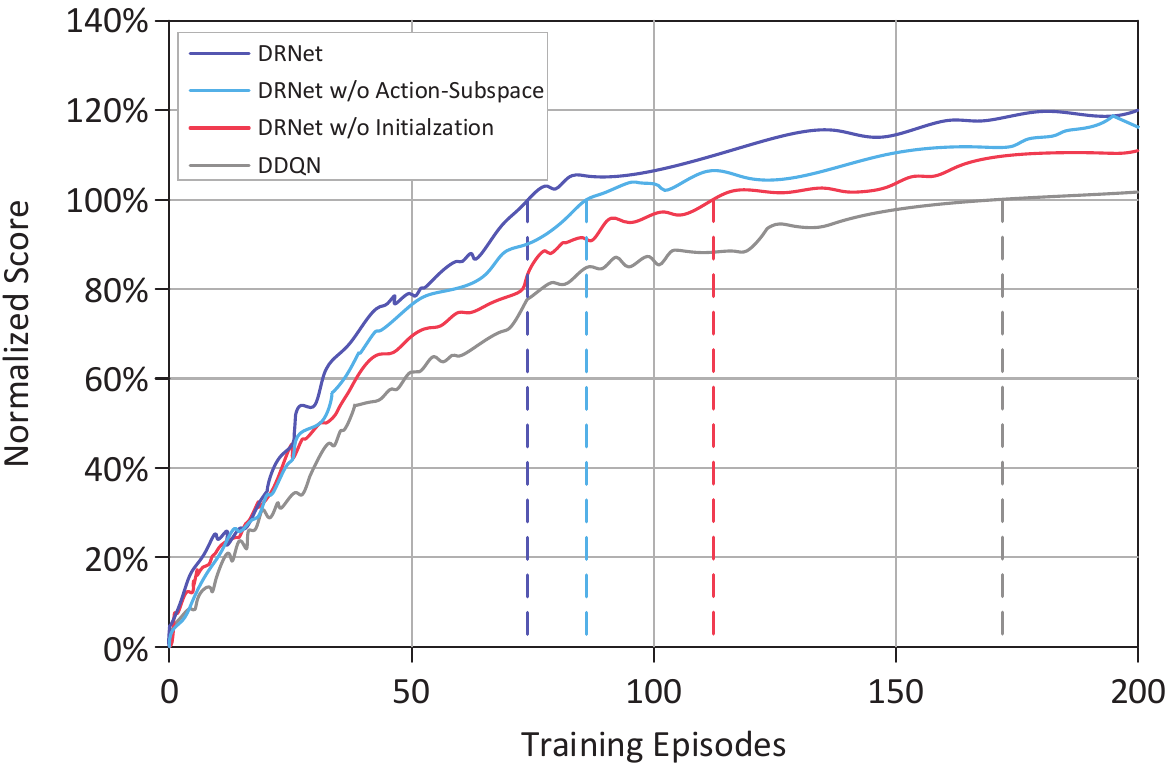}}
\caption{Summary plots of learning speed}
\label{learningrate}
\end{figure}
\subsection{Experimental Setting}
In our simulation, the size of the occupancy grid is set to 30$\times$15 and the visibility of the autonomous vehicle is 20m in front and 10m in back with a longitudinal discretization of 1m per cell (all cars are 6m in length, 3m in width), and 1 lane to the left and right with 5 one-cell-per-lane discretizations in the lateral direction. We consider the grid of 30$\times$15 as a unit and the autonomous vehicle requires a distance of 8193 meters to complete its journey in the simulation environment, which is seen as an episode of our training process. In total, our training process is composed of 200 episodes. We set traffic density to $D_{lane} = \left\lbrace (1,0.1), (2,0.3), (3,0.5)\right\rbrace $ for lanes 1 to 3. The traffic density increases from left to right, in order to give faster, sparse traffic in the leftmost lanes and slower, dense traffic in the rightmost lanes. Such traffic conditions ensure that non-trivial lane-change maneuvers, like merging and overtaking, would be necessary to remain efficient. We set $v \in [10,80] $, $v_{p} \in [20,60]$, $v_{des}=75$ and $\beta_t=0.7$. Table.\ref{setup} summarizes all experimental setting parameters used.
\begin{table}[h]
\caption{Experimental Setting Parameters}
\begin{center}
\resizebox{0.8\textwidth}{!}{
\begin{tabular}{lc}
\hline
\centering \textbf{Parameter}& \textbf{Explanation} \\
\hline
Number of lanes, $L$& 3 \\
\hline
Velocity of DRL agent, $v$ & $v\ \in [10,80]$\\
\hline
Velocity of participants, $v_p$ & $v_p \in [20,60]$\\
\hline
Desired velocity of DRL agent, $v_{des}$ & $v_{des}=75$\\
\hline
Penalty for frequent lane-changing, $\beta_t$  & $\beta_t=0.7$\\
\hline
Traffic density, $\lbrace d_1, d_2, d_3\rbrace$  & $d_1=0.1, d_2=0.3, d_3=0.5 $\\
\hline
\end{tabular}}
\label{setup}
\end{center}
\end{table}

We use the following four metrics to measure the performance: average velocity ($\overline{v}$); safety ratio ($S$: the ratio of the number of test episodes without collisions to the total number of test episodes); average lane-changing times ($\overline{Lc}$); and, most importantly, decision-making efficiency ($\sigma$), calculated as follows:
\begin{equation}
\sigma=\frac{\overline{v} \cdot S}{\overline{Lc}}\label{ef}
\end{equation}

\begin{figure*}[h]
\centering
\subfigure[]{ 
\includegraphics[width=3in]{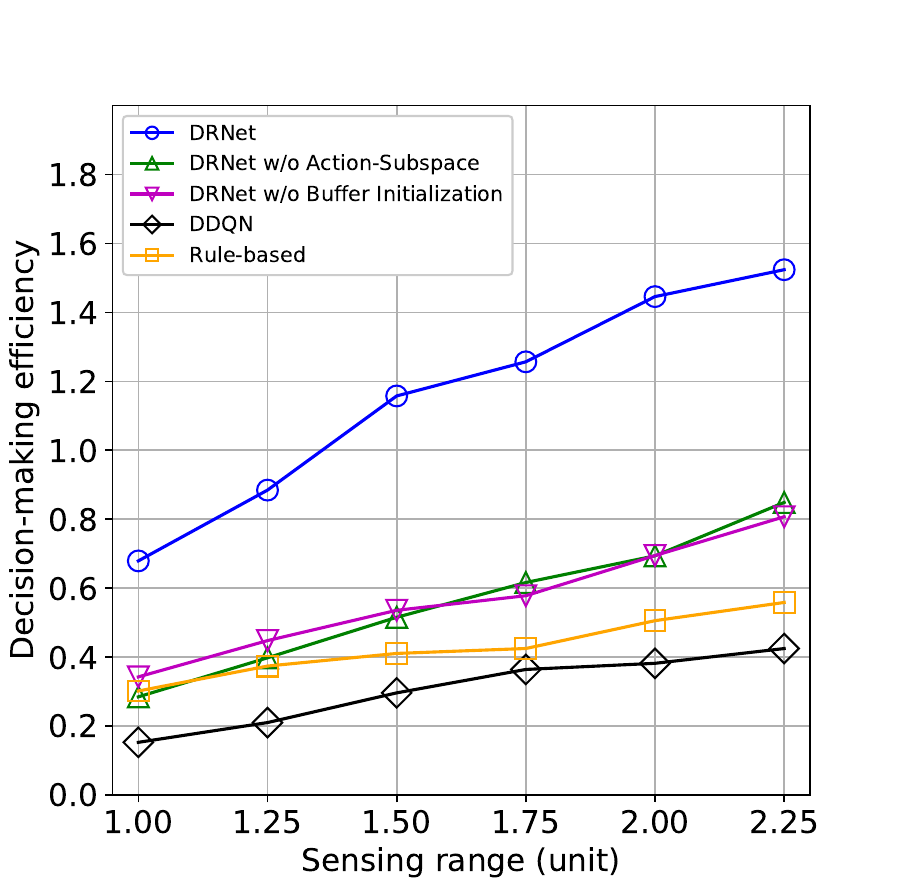}}
\subfigure[]{
\includegraphics[width=3in]{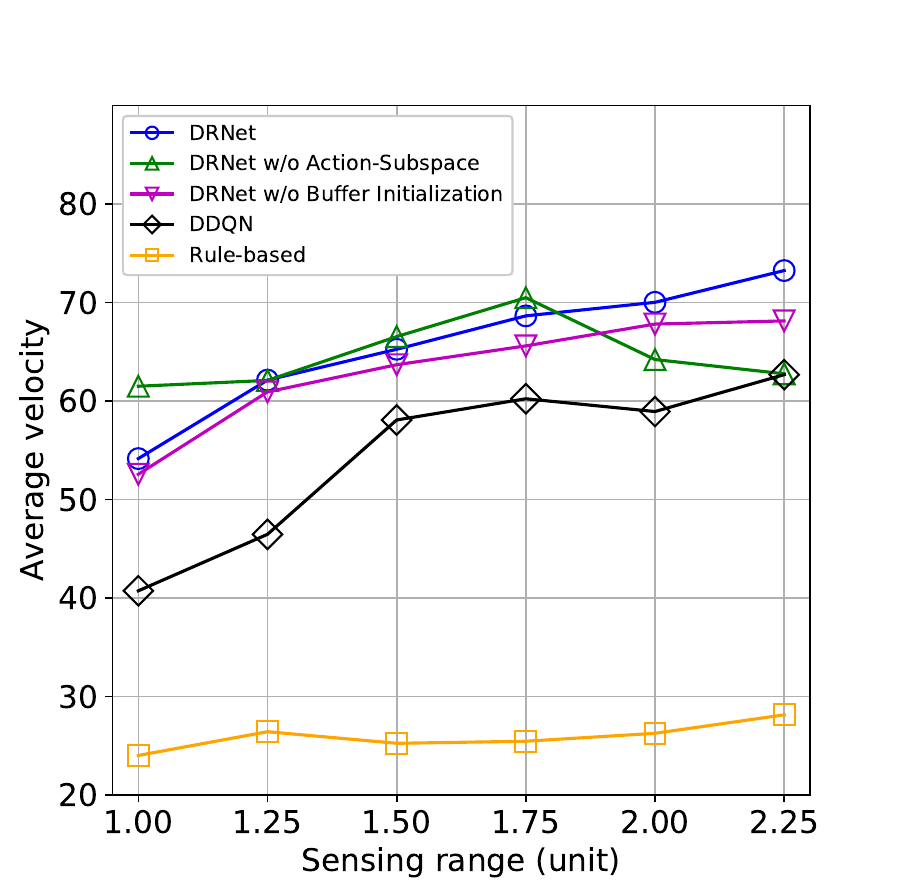}}
\subfigure[]{ 	
\includegraphics[width=3in]{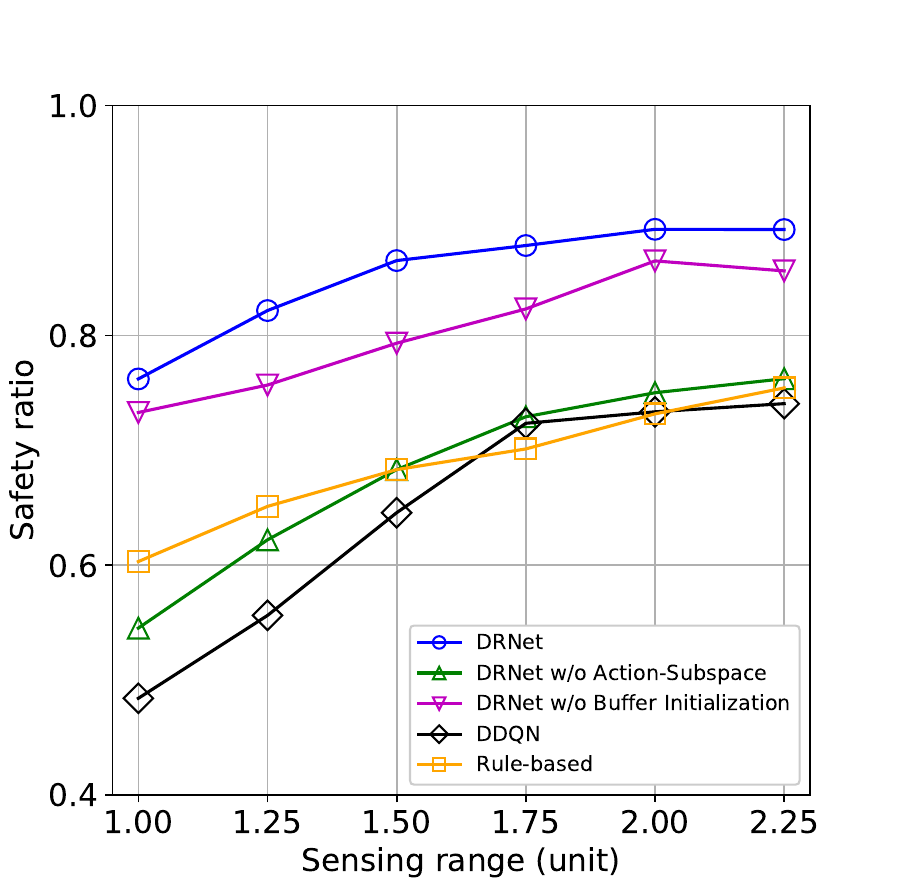}}
\subfigure[]{
\includegraphics[width=3in]{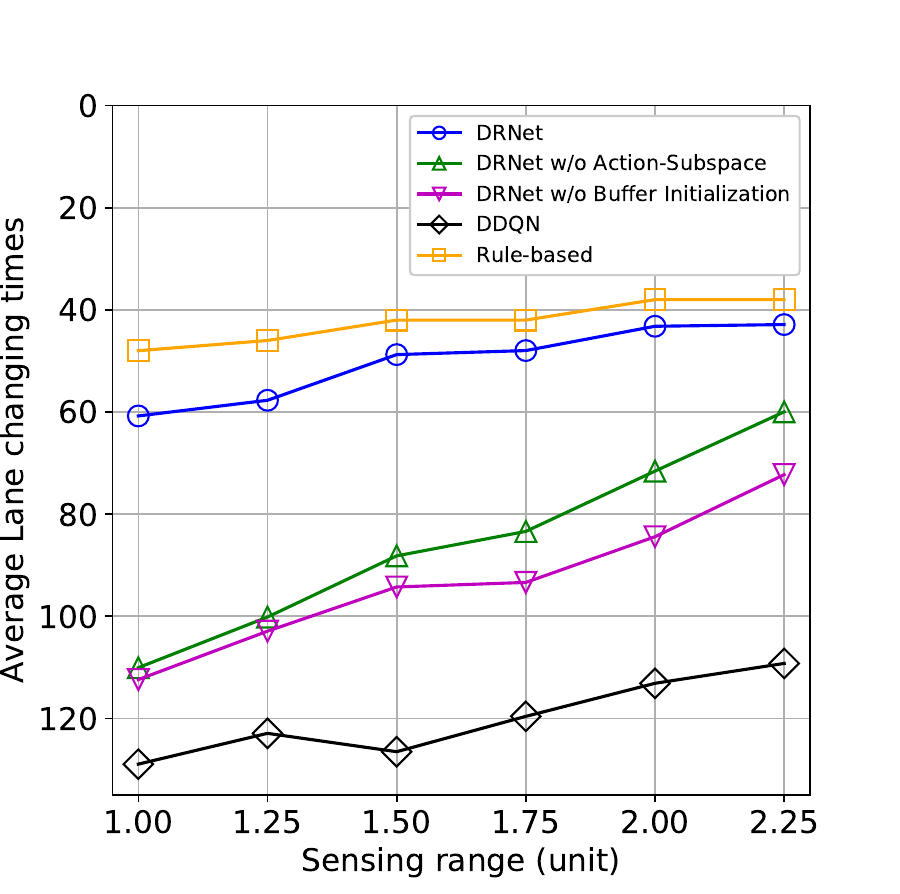}}
\caption{Impact of sensing range on four metrics. }
\label{impact}
\end{figure*}
\begin{figure*}[h]
\centering
\subfigure[]{ 
\includegraphics[width=3in]{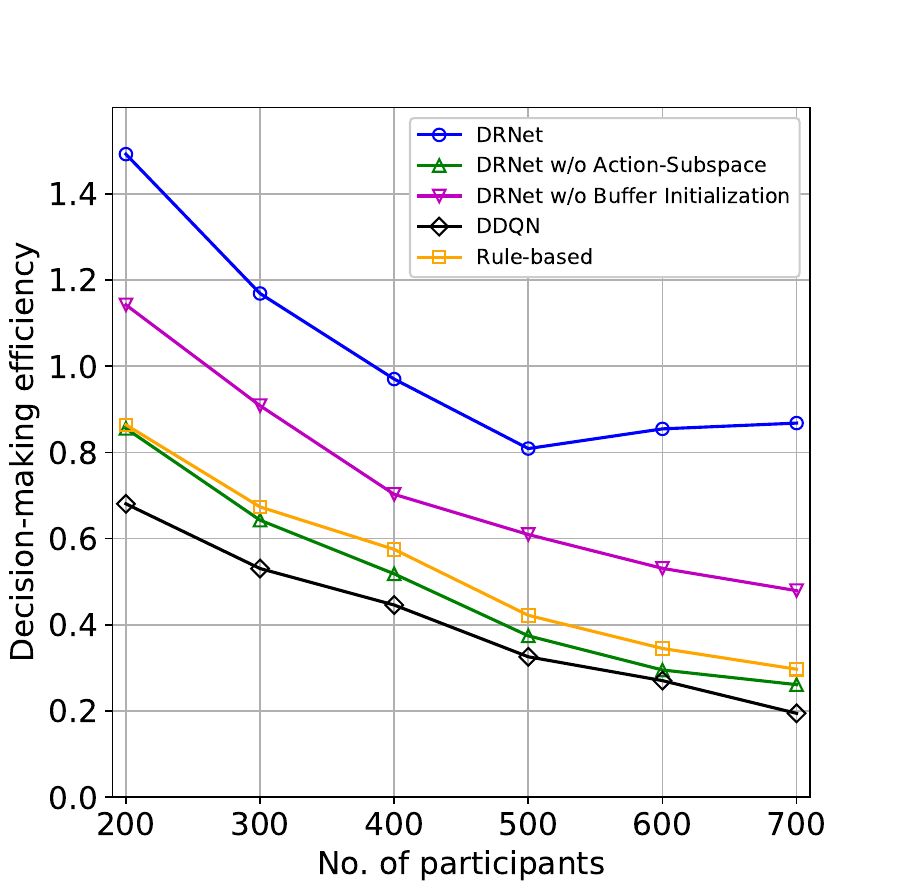}}
\subfigure[]{
\includegraphics[width=3in]{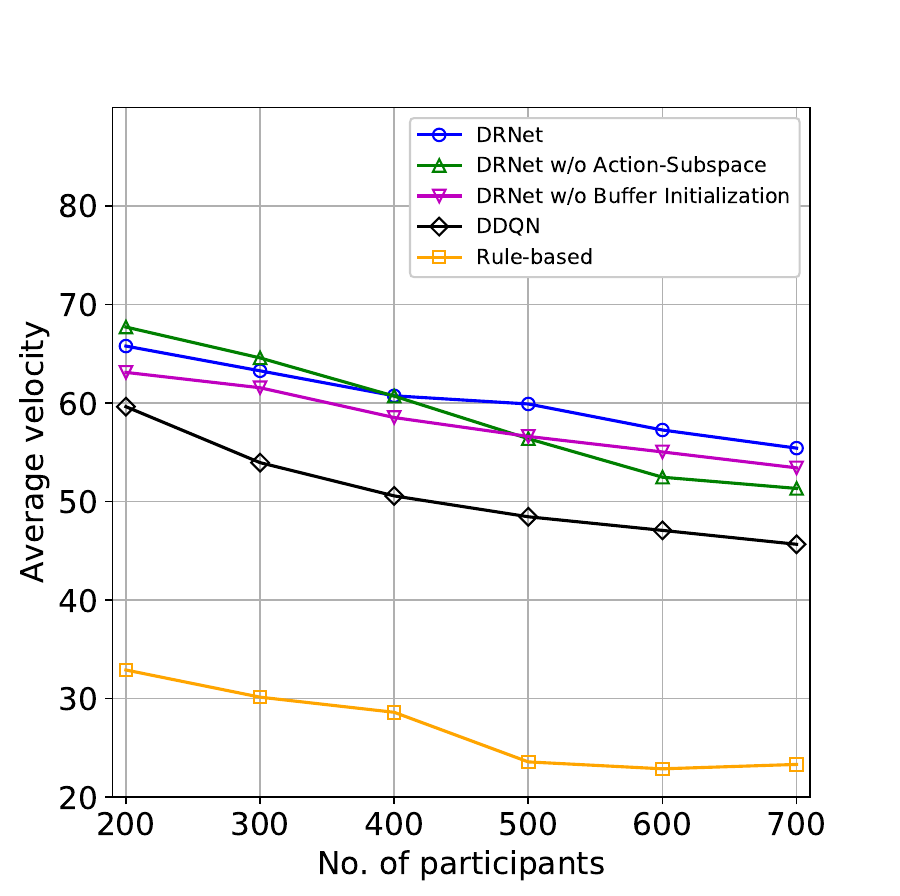}}
\subfigure[]{ 	
\includegraphics[width=3in]{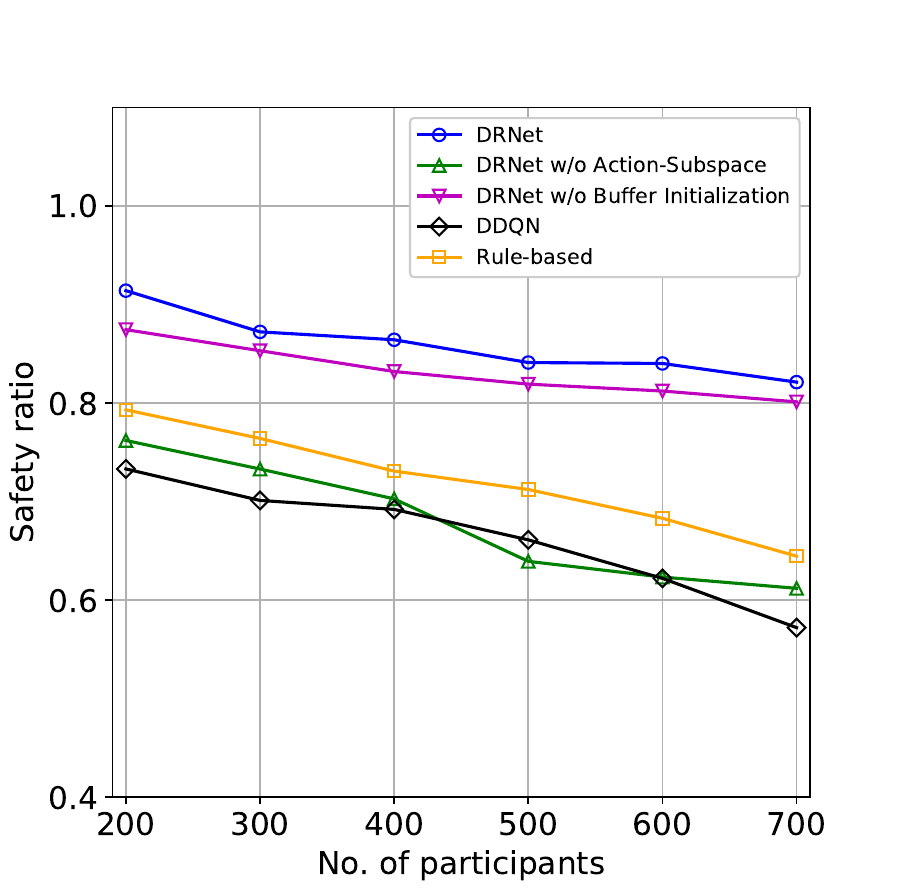}}
\subfigure[]{
\includegraphics[width=3in]{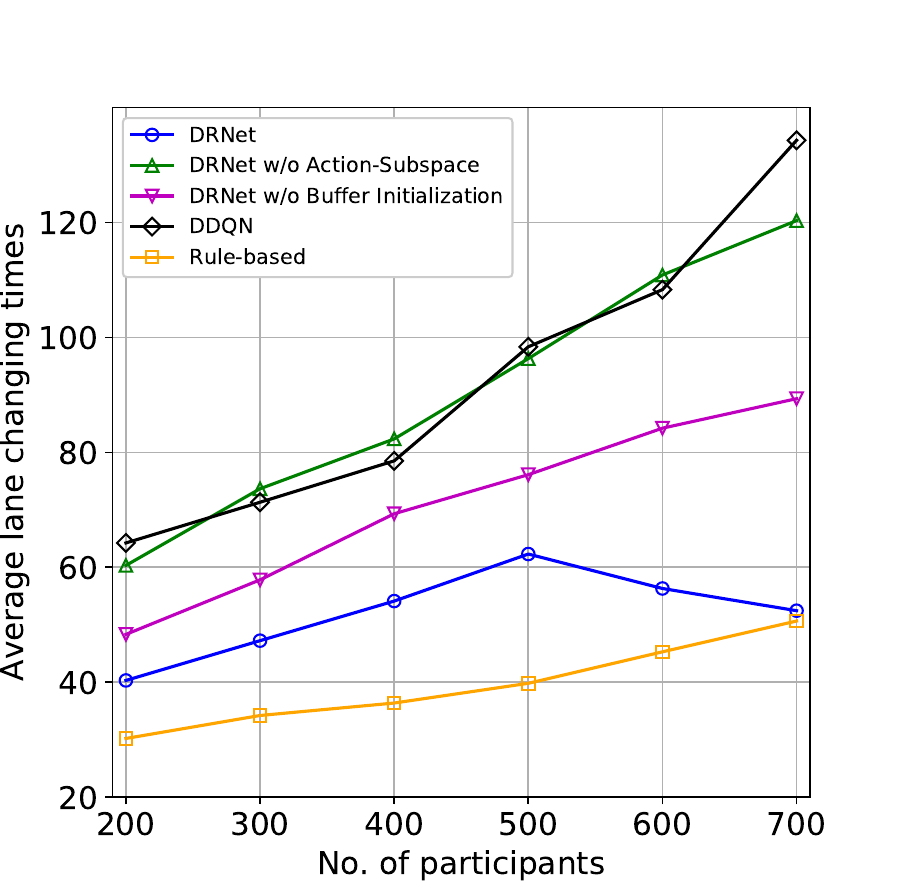}}
\caption{Impact of number of participants on four metrics. }
\label{impact2}
\end{figure*}
\subsection{State-of-the-Art Approach and Baselines}
We compared our approach DRNet with DDQN \cite{Hasselt:2016:DRL:3016100.3016191}, which is considered by DeepMind to be one of the state-of-the art DRL approaches. It is an off-policy-based approach which is introduced in a tabular setting and can be generalized to work with large-scale function approximation, while using the existing architecture and deep neural network of the DQN without requiring additional networks or parameters. In addition, we compared our approach with three other commonly used baselines:

\begin{itemize} 
\item[•]DRNet w/o Action-Subspace: In the decision-making process, its policy uses only the soft-max to select actions and store transitions, rather than using Action-Subspace (as in Section IV.C).
\item[•]DRNet w/o Initialization: During training, we only learn from tabula rasa and store experience transitions, without initialization of the replay buffer (as in Section IV.B).
\item[•]Rule-based \cite{ardelt2012highly}: The traditional and classic decision-making approach for lane changing, in which the vehicle makes a decision by the predefined rules. 
\end{itemize}

\subsection{Training Phase and Learning Evaluation}
Since the appropriate training samples are required, the agent must face the balance between exploration and exploitation. The trade-off between exploration and exploitation in this study is handled by following an $\varepsilon$-greedy policy\cite{mnih2015human}. The main idea of the $\varepsilon$-greedy algorithm is to randomly select an action from the action space with a probability of $\varepsilon$, and to select the current optimal action according to the greedy method with a probability of 1$-\varepsilon$. In this study, we set exploration rate $\varepsilon = max(\varepsilon_0\cdot\gamma^t, \varepsilon_{min})$, where $\varepsilon_0=1, \varepsilon_{min}=0.001$. To measure learning efficiency, Fig.\ref{learningrate}) shows the change in the normalized score over time during training, comparing DRNet with three metrics (excluding the rule-based one). The normalized score for each episode is derived as in \cite{Hasselt:2016:DRL:3016100.3016191}:

\begin{equation}
score=\frac{r-r_{random}}{r_{rule}-r_{random}}
\end{equation}

where the $r_{rule}$ and $r_{random}$ are the rewards in the rule-based approach and random policy, respectively. In Fig. \ref{learningrate}, the equivalence points are highlighted with dashed lines: these are the steps at which the curves reach $100\%$, (i.e., when the algorithm performs equivalently to rule-based driving in terms of score over training). DRNet gives a performance boost over agent learning and, in the aggregate, learning is twice as fast as for the typical DDQN. This is because DRNet filters the actions which are not in the safe subspace at each timeslot, as well as initializing the replay buffer and sampling important transitions.

\subsection{Finding Appropriate Hyperparameters}
Suitable hyperparameters in neural networks will significantly improve the overall performance. We now present the results of our experiments to establish appropriate hyperparameters for the DRNet. For most parameters we simply reused the same common settings as in other DRL algorithms like DQN and DDQN, setting the initial learning rate to 0.0005, update factor $\tau = 0.001$ and batch size $H=512$. Here we select the discount factor $\gamma$, the experience replay buffer size \textit{B}, the number of CNN layers and the number of neurons in fully connected layers. The results are shown in Tables \ref{tablehy1} and \ref{tablehy2}. Our evaluation metric is decision-making efficiency $\sigma$ in Eq.\ref{ef}. In Table.\ref{tablehy1}, we first show the impact of the discount factor and buffer size on decision-making efficiency $\sigma$ when the numbers of neurons and CNN layers are set to 96 and 3, respectively. Table.\ref{tablehy2} shows how the numbers of neurons and CNN layers affect $\sigma$, when discount factor and buffer size are set to 0.93 and $5\times 10^5$, respectively.
\begin{table}[h]
\caption{Impact of discount factor and buffer size on decision-making efficiency}
\begin{center}
\resizebox{0.8\textwidth}{!}{
\begin{tabular}{c|c|c|c|c|c}
\hline
\multirow{2}{*}{Discount factor $\gamma$} & \multicolumn{5}{c}{Buffer size \textit{B}}                           \\ \cline{2-6} 
                                   & {$1\times10^5$}    & \multicolumn{1}{c|}{$2\times10^5$}    & \multicolumn{1}{c|}{\bm{$5\times 10^5$}} & \multicolumn{1}{c|}{$7\times10^5$}    & \multicolumn{1}{c}{$9\times10^5$}    \\ \hline
0.81                               & 0.173265 & 0.347284 & 0.944304       & 0.63425  & 0.566738 \\ \hline
0.84                               & 0.20883  & 0.401099 & 0.920541       & 0.767624 & 0.628575 \\ \hline
0.87                               & 0.26868  & 0.424713 & 1.135139       & 0.83635  & 0.419852 \\ \hline
0.9                       & 0.582929 & 0.351242 & 1.340483       & 0.927959 & 0.743462 \\ \hline
\textbf{0.93}                               & 0.626841 & 0.576017 & 1.580042       & 1.005315 & 0.967259 \\ \hline
0.96                               & 0.486323 & 0.54218  & 1.098705       & 0.841928 & 0.473383 \\ \hline
0.99                               & 0.460551 & 0.4132   & 0.849462       & 0.656382 & 0.389205 \\ \hline
\end{tabular}}
\end{center}
\label{tablehy1}
\end{table}

\begin{table}[h]
\caption{Impact of number of neurons and CNN layers on decision-making efficiency}
\begin{center}
\resizebox{0.8\textwidth}{!}{
\begin{tabular}{c|c|c|c|c|c}
\hline
\multirow{2}{*}{Neurons} & \multicolumn{5}{c}{CNN layers}                        \\ \cline{2-6} 
                               & 2        & \textbf{3} & 4        & 5        & 6        \\ \hline
32                             & 0.373727 & 0.44171    & 0.57528  & 0.48035  & 0.235658 \\ \hline
64                             & 0.854187 & 1.324332   & 1.145406 & 0.757247 & 0.395925 \\ \hline
\textbf{96}                    & 0.930803 & 1.819101   & 1.466221 & 1.383737 & 1.17227  \\ \hline
128                            & 0.791346 & 1.250229   & 1.252174 & 1.181278 & 0.995996 \\ \hline
160                            & 0.906912 & 1.796139   & 1.483087 & 1.275611 & 1.254938 \\ \hline
192                            & 0.616605 & 1.504122   & 1.103096 & 1.238374 & 1.067722 \\ \hline
224                            & 0.595408 & 1.475829   & 1.057635 & 1.085019 & 1.023103 \\ \hline
\end{tabular}}
\end{center}
\label{tablehy2}
\end{table}

\begin{figure*}[t]
\centering
\subfigure[step 1]{ 
\includegraphics[width=3in]{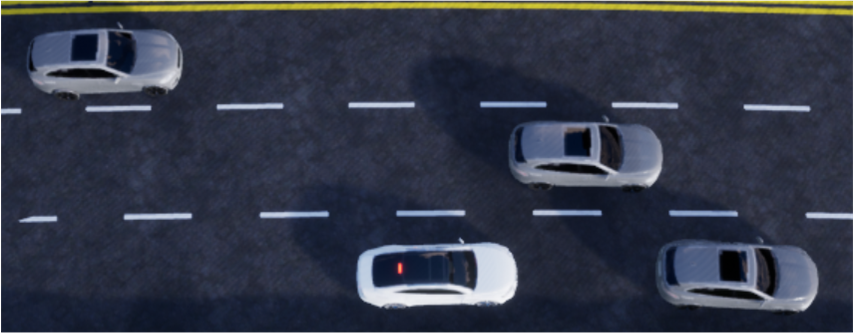}}
\subfigure[step 2]{
\includegraphics[width=3in]{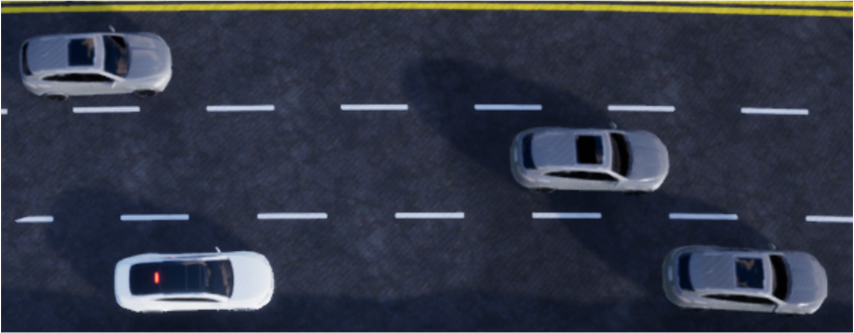}}
\subfigure[step 3]{ 	
\includegraphics[width=3in]{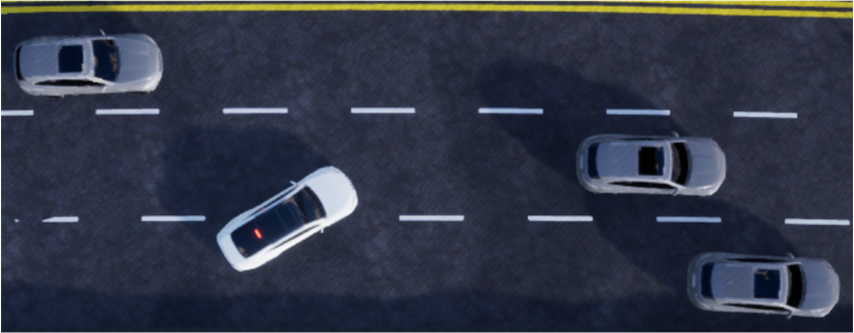}}
\subfigure[step 4]{
\includegraphics[width=3in]{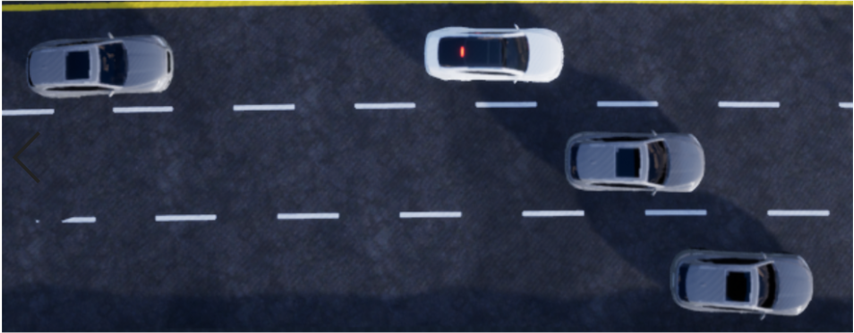}}
\caption{Case1: Far-sighted behavior}
\label{case1}
\end{figure*}
\begin{figure*}[t]
\centering
\subfigure[The original situation]{ 
\includegraphics[width=3in]{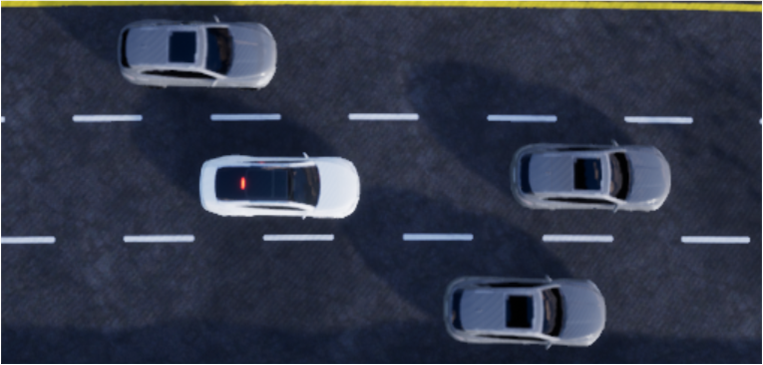}
\label{case21}
}
\subfigure[DRNet]{
\includegraphics[width=3in]{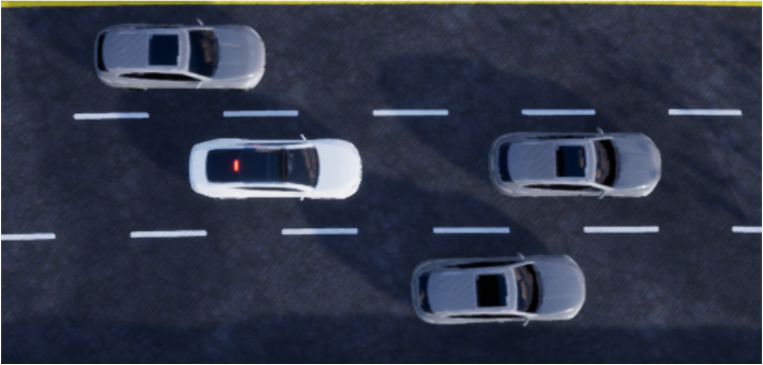}\label{case22}}
\subfigure[DDQN]{ 	
\includegraphics[width=3in]{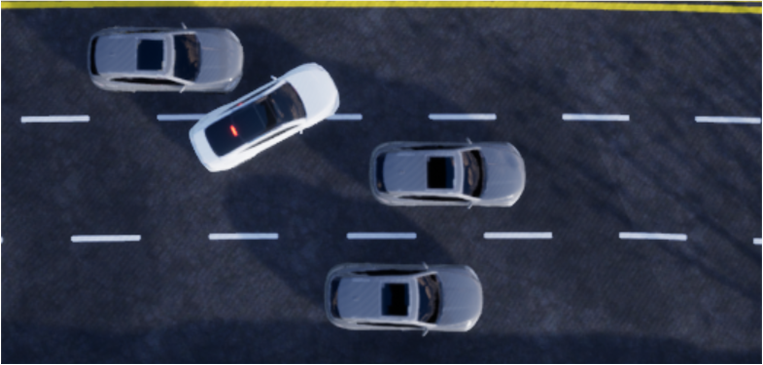}}
\subfigure[rule-based]{
\includegraphics[width=3in]{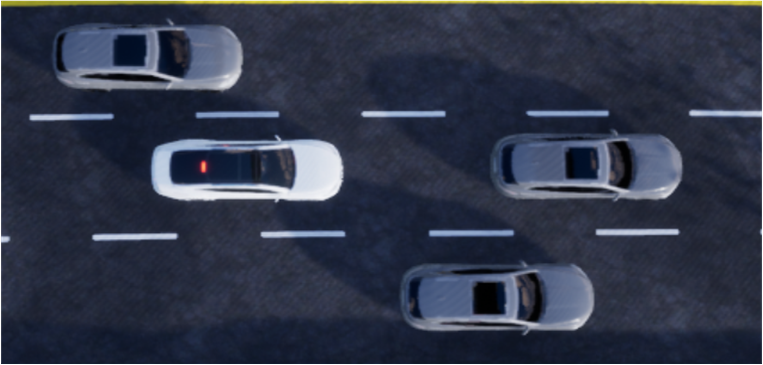}\label{case23}}
\caption{Case2: Impact of Action-Subspace}
\label{case2}
\end{figure*}
\subsection{Comparison with State-of-the-Art Approach and Baselines}
We conducted simulations by varying the sensing range $R$ and the number of participants. Other hyperparameters in our model include the discount factor $\gamma=0.93$, buffer size $B=5\times10^5$, number of CNN layers (3 here) and number of neurons (96). We first examine the impact of sensing range within $[1,2.25]$ with a step of 0.25 on four metrics, as shown in Fig.\ref{impact}. We changed the sensing range from $U=1.00$ to $U=2.25$ with a step of 0.25. After that, Fig.\ref{impact2} shows the impact of number of participants: sensing range $U$ was set to 1.75, and number of participants $P$ was changed from 200 to 700.

From Figures \ref{impact} and \ref{impact2}, we can make the following observations:

(1) First, our method consistently outperforms all the baselines in terms of decision-making efficiency $\sigma$. For example, in Fig.\ref{impact}(a), when the sensing range is 1.50, DRNet achieves a decision-making efficiency of 1.16, compared to 0.54 given by the best baseline DRNet w/o Initialization, a nearly threefold improvement (x2.91). On average, for $\sigma$, DRNet yields significant improvement: by factors of 1.13, 1.04, 2.89 and 1.66 over DRNet w/o Action-Subspace, DRNet w/o Initialization, DDQN and Rule-based, respectively. 

(2) It can be seen from Figures \ref{impact}(a) and \ref{impact2}(a) that the decision-making efficiency of DRNet increases monotonically with the sensing range, and decreases with number of participants from a global perspective. This is because the larger sensing range represents better collection of traffic information, which could make the average velocity and safety ratio keep increasing (as shown in Figures \ref{impact}(b) and \ref{impact}(c)), while reducing unnecessary lane-changing times (Fig.\ref{impact}(d)), so that better decision-making efficiency is achieved. Meanwhile, the larger number of participants means more complicated traffic information and more lane-changing decisions (as shown in Fig.\ref{impact2}(c)).

(3) Finally, it can be seen that DRNet always outperforms DDQN on the four metrics. For instance, in Fig.\ref{impact}(a-d), when the sensing range is 2, DRNet increases the four metrics($\sigma, \overline{v}, S, \overline{Lc}$) by 2.78 times, 18.81$\%$, 21.65$\%$, 61.79$\%$, respectively. DRNet is trained with more “useful” sample transitions, and has a more accurate value estimation of state and action pairs than DDQN. Moreover, DDQN leads to much unsafe decision-making (low safety ratios and high lane-changing times – see Figures \ref{impact}(c-d) and \ref{impact2}(c-d)) due to its blind pursuit of reward maximization.
\subsubsection{Case Study}
Figures \ref{case1} and \ref{case2} show the performance of the trained DRL agent in two segments of one of the testing simulations. In Fig.\ref{case1}, the DRNet agent (white car) executes two lane changes in order to overtake slower leading vehicles and avoid getting “trapped". Starting from (a), the agent observes the slower vehicles in front and in the adjacent lanes; (b) decelerates to prepare for the lane change; (c) executes a lane change to the leftmost lane at a safe distance from the vehicle behind; (d) continues in the leftmost lane and accelerates. In Fig.\ref{case2}, the DRNet agent validates the impact of Action-Subspace in comparison with other methods. In Fig.\ref{case21}, we can see that there are vehicles both in front of the agent and in the adjacent lanes (on both sides). In Fig.\ref{case22}-\ref{case23}, DRNet and the rule-based agent keep following the leading vehicle due to the unsafe distance from the rear left vehicle, while the DDQN agent adopts radical behavior to get more rewards. This decision of the DDQN agent may obviously lead to a collision should the rear vehicle suddenly accelerate.
\subsection{Real-world Study}
To further evaluate the effectiveness of DRNet's lane-change decision-making, data from the Federal Highway Administration’s Next Generation Simulation (NGSIM) program were used \cite{punzo2011assessment}. The NGSIM dataset includes data for two sections of highway, on the Interstate 80 freeway and U.S. Highway 101. The I-80 (BHL) test section is a 0.40 mile (640m) 6-lane freeway system test section with weaving zones and an HOV lane. Processed real data include 45 minutes of vehicle trajectories in the transition (4:00-4:15 pm) and congestion (5:00-5:30 pm) phases. The US101 site is a 0.3 mile (500 m) weaving test section with a five-lane system. Processed real data include 45 minutes of vehicle trajectories in the transition (7:50-8:05 am) and congestion (8:05-8:35 am) phases. 

From the NGSIM dataset, 200 data points were selected as a validation, of which half were chosen for non-merge events (lane stay) and the other half for merge events (lane change). We predicted lane-change intention with DRNet and compared it with the actual situation. The prediction results are shown in Table.\ref{real}. Through analysis of the wrongly predicted cases, we find that our model tends to be a relatively “sharp-witted driver". Compared with conservative drivers, it can seize the opportunity to change lanes, and yet drive more safely than aggressive drivers.
\begin{table}[]
\caption{Prediction Results of DRNet on Real-world Datasets}
\centering
\begin{tabular}{|c|c|c|}
\hline
\multirow{2}{*}{Decision-making} & \multicolumn{2}{c|}{Validation data} \\ \cline{2-3} 
                                 & Observation          & DRNet         \\ \hline
Non-merge (lane stay)            & 100                  & 87\%          \\ \hline
Merge (lane change)              & 100                  & 96\%          \\ \hline
\end{tabular}

\label{real}
\end{table}
\section{Conclusions}
DRNet is a tactical decision-making framework that combines the advantages of both DRL and rule-based methods while avoiding their limitations.  The neural architecture of DRNet is developed to work seamlessly with a lane-changing objective, including experience replay initialization and Action-Subspace, which integrates ideas from safety verification. Last, we find a set of hyperparameters — ie., discount factor 0.93, experience replay buffer size 5$\times 10^5$, 3 CNN layers and 96 neurons of fully-connected layers — for best performance. Compared with the state-of-art approach DDQN and three other baselines, DRNet shows significantly superior performance in terms of average velocity, safety ratio, average lane-changing times and decision-making efficiency.

\bibliographystyle{unsrtnat}
\bibliography{drnet}  






\end{document}